\def\year{2018}\relax
\documentclass[letterpaper]{article}
\usepackage{aaai18}
\usepackage{times}
\usepackage{helvet}
\usepackage{courier}
\usepackage{amsfonts}

\usepackage{graphicx}
\usepackage{subfigure}
\usepackage{bm}
\usepackage{url}            % simple URL typesetting
\usepackage{multirow}
\usepackage{rotating}
\usepackage{amsmath}
\usepackage{verbatim}
\usepackage{algorithm}
\usepackage{algorithmic}
\usepackage{comment}

\usepackage{enumitem}
\usepackage{amsmath,amssymb,amsthm}

\title{Tree-Structured Neural Machine\\ for Linguistics-Aware Sentence Generation}

\newtheorem{myApt}{Assumption}
\newtheorem{myTheo}{Theorem}
\newtheorem{myDef}{Definition}

\author{Ganbin Zhou$^{1,2}$, Ping Luo$^{1,2}$, Rongyu Cao$^{1,2}$, Yijun Xiao$^{3}$, Fen Lin$^{4}$, Bo Chen$^{4}$, Qing He$^{1,2}$ \\
	$^{1}$Key Lab of Intelligent Information Processing of Chinese Academy of Sciences (CAS), \\Institute of Computing Technology, CAS, Beijing 100190, China.\{zhouganbin, luop, heqing\}@ict.ac.cn\\
	$^{2}$University of Chinese Academy of Sciences, Beijing 100049, China. \\
	$^{3}$Department of Computer Science, University of California Santa Barbara, Santa Barbara, CA 93106, USA. \\
	$^{4}$WeChat Search Application Department, Tencent, China. \\
}

\begin{document}
	\maketitle
	\begin{abstract}
		Different from other sequential data, sentences in natural language are structured by linguistic grammars. Previous generative conversational models with chain-structured decoder ignore this structure in human language and might generate plausible responses with less satisfactory relevance and fluency. In this study, we aim to incorporate the results from linguistic analysis into the process of sentence generation for high-quality conversation generation. Specifically, we use a dependency parser to transform each response sentence into a dependency tree and construct a training corpus of \textit{sentence-tree} pairs. A tree-structured decoder is developed to learn the mapping from a sentence to its tree, where different types of hidden states are used to depict the local dependencies from an internal tree node to its children. For training acceleration, we propose a tree canonicalization method, which transforms trees into equivalent ternary trees.
		Then, with a proposed tree-structured search method, the model is able to generate the most probable responses in the form of dependency trees, which are finally flattened into sequences as the system output.
		Experimental results demonstrate that the proposed \textsc{X2Tree} framework outperforms baseline methods over 11.15\% increase of acceptance ratio.
	\end{abstract}
	
	\section{Introduction}
	Many natural language processing tasks can be formulated as sequence to sequence problems. Given a sequence of tokens, this task is to generate another sequence of tokens of equal or non-equal length. For example, machine translation models try to find a sequence of words in the target language, expressing the identical meaning to a source sentence; conversational models respond to a post utterance with a semantically coherent and grammatically correct sentence. Neural models were applied to these tasks and achieved state-of-the-art performances in recent years ~\cite{Cho2014,Shang2015,Vinyals2015ncm,Sordoni2015,Serban2015}.
	
	These neural models in essence use a chain-structured decoder to sequentially generate tokens given a context vector encoded from an input sequence. We notice that this decoding process is mostly linear, meaning that tokens are obtained in the order of their appearances. It basically considers the dependency between any word and all its preceding ones. RNN models, such as LSTM~\cite{Hochreiter1997} and GRU~\cite{Cho2014}, are developed in demand of capturing both the short and long-distance dependency over the chain structure.

	Our work improves upon these studies by incorporating the results from linguistic analysis into the decoder. Specifically, we leverage a dependency parser to transform each response sentence into a dependency tree, containing more local dependency information. The proposed model learns to map a sentence into a canonicalized tree, which is then flattened as the final output.
	%Consider the example in Fig.~\ref{fig:dependency_parser}, which is an intermediate task for automatic conversation generation.
	Consider the intermediate task for automatic conversation generation.
	Instead of generating the response to a given input post directly, we aim to generate the \emph{dependency parse tree} of the corresponding response in top-down fashion.
	Additionally, a tree canonicalization method is proposed, aiming at transforming trees with different numbers of children nodes into their equivalent form, namely full ternary trees, in order to accelerate training and simplify model implementation  on GPU.
	Then, a postprocessing step converts the dependency tree into a sequence as the final response. We also theoretically prove that ternary tree is the ``best" choice for model complexity.

	Some models also process trees in a \emph{bottom-up} fashion.
	Socher et al.~\shortcite{DBLP:conf/icml/SocherLNM11} proposed a max-margin structure prediction architecture based on recursive neural networks, and demonstrated that it successfully parses sentences and understands scene images. Tai et al. ~\shortcite{Tai2015} and Zhu et al.~\shortcite{Zhu2015Long} extended the chain-structured LSTM to tree-structured LSTM, which is shown to be more effective in representing a tree structure as a latent vector. All these models process trees in a \emph{bottom-up} fashion, where children nodes are recursively merged into parent nodes until the root is generated.
	
	%, and some of them are merged to generate an internal node labeled with proper categories of the grammar
	
	However, bottom-up models require all the leaf nodes in the predicted tree given in advance. For example, to generate the \emph{constituency parse tree} for a sentence (shown in Fig.~\ref{fig:constituency_parser}), tokens appeared in the given sentence are used as leaf nodes in this tree. Similarly, to parse natural scene images~\cite{DBLP:conf/icml/SocherLNM11}, an image is first divided into segments, each of which corresponds to one leaf node in output tree. With these given leaves  bottom-up process recursively processes the internal nodes until the root is built.
	
	\begin{figure}[htbp]
		\centering
		\subfigure[Constituency parser in bottom-up fashion.]{
			\label{fig:constituency_parser}
			\includegraphics[width=2.5in]{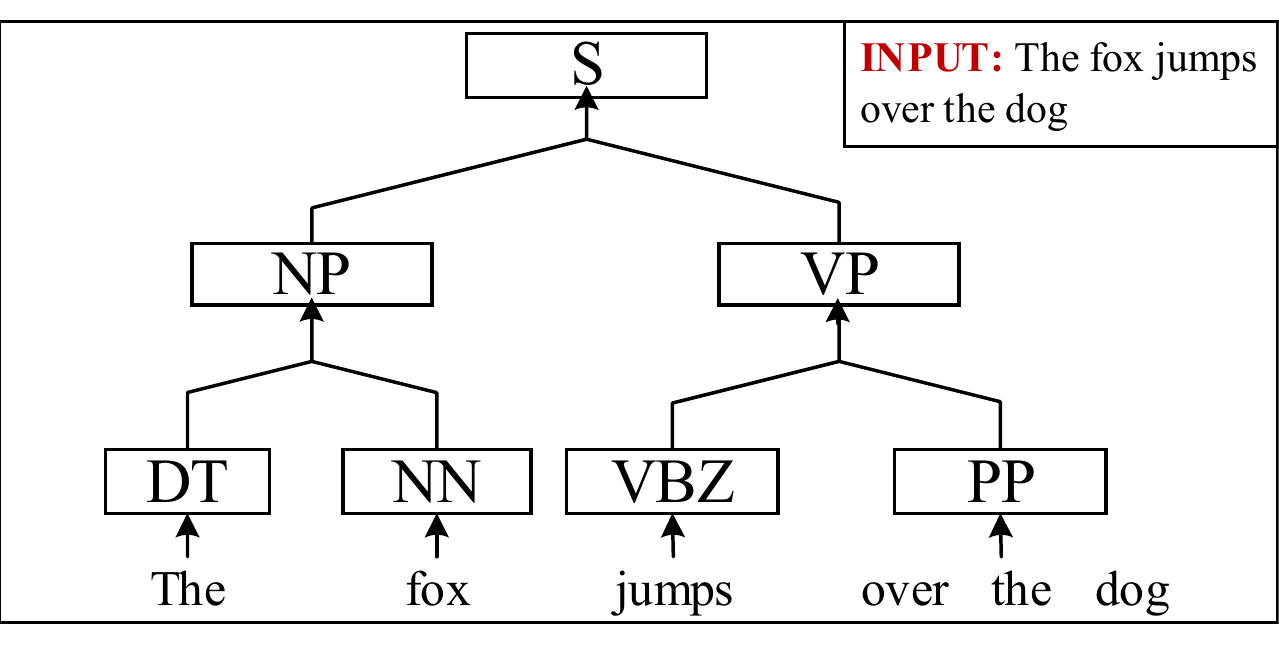}
		}
		\subfigure[Dependency parser in top-down fashion.]{
			\label{fig:dependency_parser}
			\includegraphics[width=2.5in]{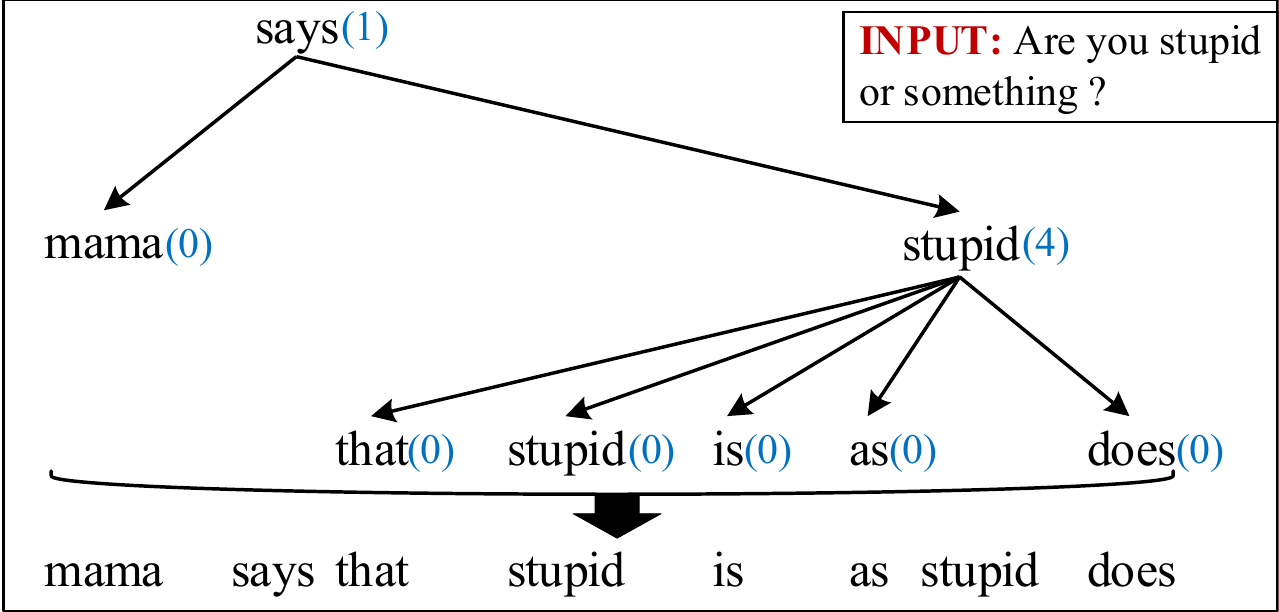}
		}
		\caption{Examples of two tree-structured prediction tasks in language understanding.}
		\label{fig:twoparser}
	\end{figure}
	
	Here, we argue that the bottom-up generative models may not work well when the leaf nodes are not specified ahead of prediction. Consider the task in Fig.~\ref{fig:dependency_parser}, which is an intermediate task for automatic conversation generation. Instead of generating the response to a given input post directly, we aim to generate the \emph{dependency parse tree} of the corresponding response. Then, a postprocessing step converts the dependency tree into a sequence as the final response. \footnote{The motivation of this solution is detailed in Section ~\emph{Tree Generation}.} Compared to the \textsc{Seq2Seq} solution to conversation generation, we argue that this tree-structured modeling method is more effective due to a shorter average decoding length and the extra structure information provided from the parse tree. In this task, it is clearly seen that: since all the tokens in response are not explicitly given by the input post, it may not be appropriate to generate the dependency from bottom to top.

	%\begin{wrapfigure}{r}{0cm}
	%	\centering
	%	\includegraphics[width=2.5in]{figs/dependency_parser.pdf}
	%	\caption{Intermediate sentence gerneration.}
	%	\label{fig:dependency_parser}
	%\end{wrapfigure}
	
	Previous works on tree-structured LSTM \cite{Tai2015,Zhu2015Long,Zhang2016Top} show that incorporating syntactic structures into the encoder or decoder results in sentence embedding with improved performances on tasks like sentiment analysis and semantic relatedness. In this paper, we propose to inject tree structures into the decoding process with the following motivations:
	
	%Compared to the \textsc{Seq2Seq} solution to conversation generation, we argue that this tree-structured modeling method is more effective due to a shorter average decoding length and the extra structure information provided from the parse tree.
	%Let $\bm{y}$ be the response sentence to an input $\bm{x}$, and $\mathcal{T}_{\bm{y}}$ be the dependency tree for $\bm{y}$. Then, the average length between any node in $\mathcal{T}_{\bm{y}}$ to its root is $O(\sqrt{|\bm{y}|})$, much smaller than the sentence length $|\bm{y}|$. Since the generation of a tree node only depends on its ancestor nodes, this tree transformation may alleviate the long-distance gap in sequence generation, and hopefully output more relevant and fluent responses.
	
	1) Dependency tree parsing extracts short-distance dependencies in the local area of a sentence. Utilizing these linguistic results reduces the difficulty in sequential learning, thus helps decoders to generate grammatically and semantically correct utterances. Let $\bm{y}$ be the response sentence to an input $\bm{x}$, and $\mathcal{T}_{\bm{y}}$ be the dependency tree for $\bm{y}$. Then, the average length between any node in $\mathcal{T}_{\bm{y}}$ to its root is $O(\sqrt{|\bm{y}|})$ \cite{Flajolet1982}, much smaller than the sentence length $|\bm{y}|$. Thus, this tree transformation may alleviate the long-distance gap in sequence generation.
	2) Words in higher levels of the dependency tree usually are more influential for the sentence. By generating more ``important'' words in earlier stages of the decoding process, we essentially free the decoder from the burden to store important semantic information for many time steps.
	3) We also believe that the process of tree-structured sentence generation is more consistent with how human construct sentences. Although people speak a sentence in a sequential order, they may keep some keywords, such as verbs and nouns, in mind before filling in more descriptive adjectives and adverbs to generate a full sentence.

	In this paper, we develop a tree-structured decoder in the framework of ``\emph{X} to tree" (\textsc{X2Tree}) learning, where \emph{X} represents any structure (e.g. chain, tree) encoding the post as a latent vector. Since all the tokens in the response are not explicitly given by the input post, it is appropriate to generate the dependency from top to bottom. To this end, we need to address the following challenges:
	
	1) We need to carefully model the different dependencies between a tree node and its children. Children at different positions may have different meanings, and the generation of a child node depends on not only its parent and ancestors but also its siblings. Thus, we need to fully consider the memory inherited from both its ancestors and siblings (detailed in Section \emph{Generative Model for $K$-ary Full Tree}).
	
	2) A tree node could obtain any number of children. It is
	%inappropriate and
	non-trivial to automatically determine the number of children.
	Furthermore, GPU-based parallel computing is difficult when the children number is different for each node. We therefore need a tree canonicalization process, which outputs an \emph{equivalent} standard tree, where each internal node has a fixed number of child nodes
	%\comment{. Also, it is required to support some postprocessing needed}
	(detailed in Section \emph{Tree Canonicalization}).

	3) In model inference, it is required to develop a algorithm searching for the most probable trees instead of sequences. Since the beam search utilized by previous studies only handles chain structures, a more general search algorithm for tree structures needs to be developed (detailed in Section \emph{Tree Generation}).

	With all these challenges addressed, our main contributions are twofold: 1) We propose a  generative neural machine for tree structures, and apply it to conversational model. Specifically, we introduce a tree canonicalization method to standardize the generative process and a greedy search method for tree structure inference. 2) We empirically demonstrate that the proposed method successfully predict the dependency trees of conversational responses to an input post. Specifically, for the task of automatic conversation the proposed \textsc{X2Tree} framework achieves 11.15\% increase of acceptance ratio.
	
	It is also worth mentioning that we do not need a perfect dependency parser. In our task, the sequential sentence is the final output, while the dependency tree is only the immediate result. If the parsed tree contains errors in similar patterns, the model can learn these patterns. After we convert the generated tree into a sequence, the sequence may be still correct, which is also demonstrated by the experiments.

	\section{X2Tree Neural Network}
	In this section, we introduce the \textsc{X2Tree} learning framework. The training dataset is given as:
	$
	D=\{(\bm{x}, \mathcal{T}_{\bm{y}})\}
	$
	where $\bm{y}$ is the response of the post $\bm{x}$ and $\mathcal{T}_{\bm{y}}$ is the corresponding tree of $\bm{y}$, e.g. dependency tree.
	%It contains pairs of $(\bm{x}, \mathcal{T}_{\bm{x}})$ and the
	Our task is to learn the mapping from $\bm{x}$ to a tree structure $\mathcal{T}_{\bm{y}}$. Specifically, it adopts the encoder-decoder framework. We assume $\bm{x}$ has already been encoded as a latent vector \big(see e.g.~\cite{Sutskever2014,Zhu2015Long}\big), and mostly focus on the tree-structured decoder for the generation of $\mathcal{T}_{\bm{y}}$.

	%\begin{figure}[t]
	%	\centering
	%	\includegraphics[width=2.5in]{figs/sharing_parameters.pdf}
	%	%\caption{$\{f_k\}_{k=1}^K$, $\{g_k\}_{k=1}^K$ and $g_{\textrm{r}}$ are shared for different nodes (orange triangles).}
	%	\caption{The top-down tree-structured generation with shared parent-children dependency. The parameters for inferencing children of each node (red triangles) are shared.}
	%	\label{fig:sharing}
	%\end{figure}
	
	As aforementioned, the developed decoder adopts a top-down generative process.
	%as shown in Fig.~\ref{fig:sharing}.
	The atom step is generating the children for a given node.
	%in this process is that: for a given node it generates its immediate children.
	This atom step is performed on each node until it cannot generate any valid nodes. 
	%The whole process recursively performs this atom step on each node until it cannot generate any valid nodes. 
	Thus, the key to the decoder is modeling the parent-children dependency.
	%, shown in the red triangles in Fig.~\ref{fig:sharing}.
	Note also that the model parameters for parent-children dependency are shared for all the atom steps in tree generation.

	\begin{figure}
		\centering
		\includegraphics[width=2.5in]{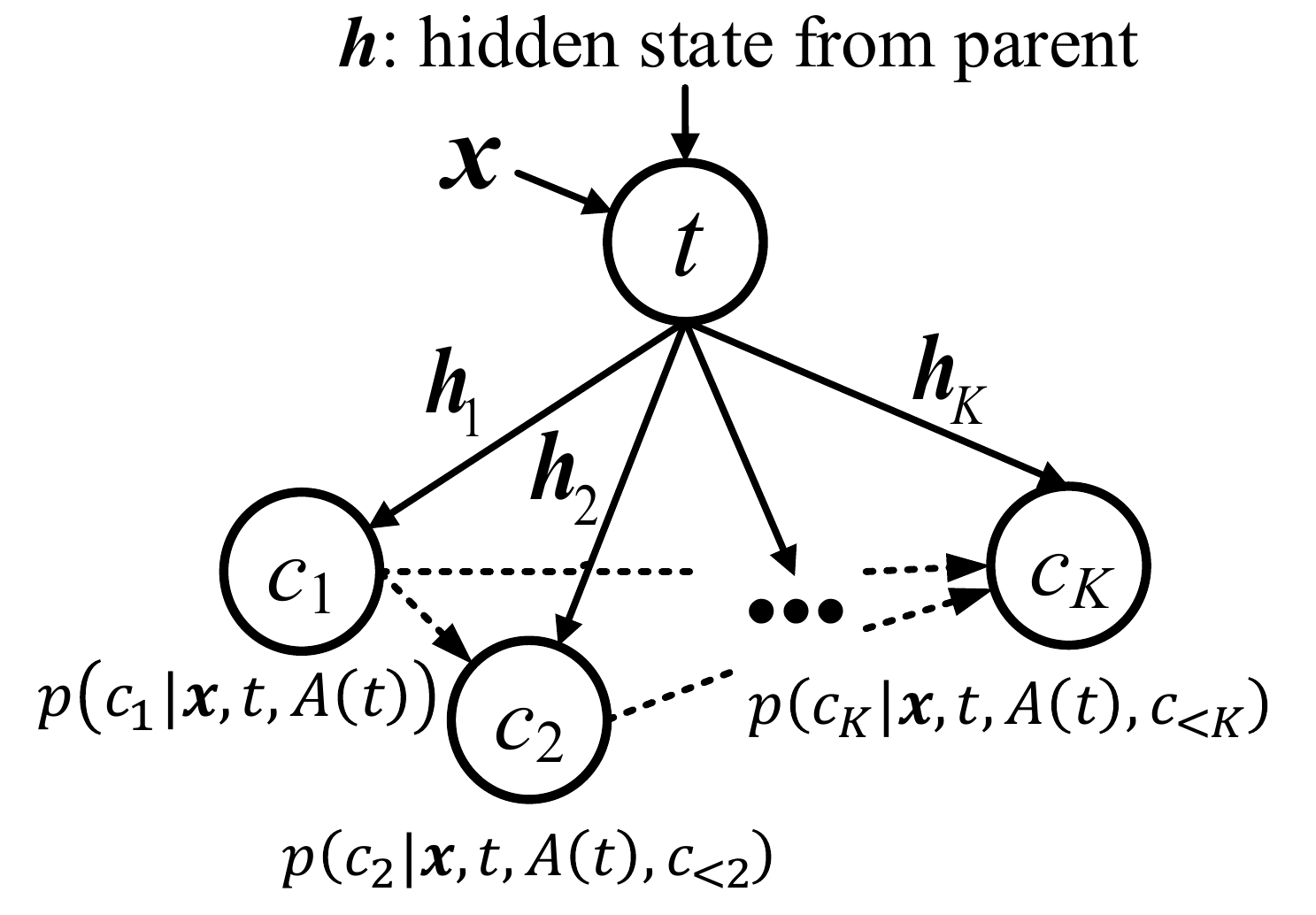}
		\caption{Parent-children dependency.}
		\label{fig:model_graph}
	\end{figure}

	We first assume the tree is $K$-ary full tree where every internal node has exactly $K$ children,
	and model  this type of tree in Section \emph{Generative Model for $K$-ary Full Tree}. Then, we propose a canonicalization method that transforms any tree into a $K$-ary full tree and discuss the $K$ for different applications in Section \emph{Tree Canonicalization}. Finally, we introduce an  algorithm for tree inference in Section \emph{Tree Generation}.

	%detail the modeling of the parent-children dependency.

	%The \textsc{X2Tree} neural network can model the probability of a $K$-ary full tree $\mathcal{T}$ conditioned on the input $\bm{x}$, namely map input $\bm{x}$ to its corresponding tree $\mathcal{T}$.
	%A dataset $D=\{(\bm{x}, \mathcal{T}_{\bm{x}})| \mathcal{T}_{\bm{x}}\text{ is corresponding tree of }\bm{x}\}$ are given for training, validating and testing, where $\bm{x}$ may be any type of data structure.
	%Here, we assume that the input $\bm{x}$ has been encoded to a context vector $\bm{x}$,
	%and focus on developing a tree-structured decoder given context vector $\bm{x}$

	\subsection{Generative Model for $K$-ary Full Tree}\label{sec:pcdm}
	
	Here, we propose a generative model for $K$-ary full tree. For simplicity, $\bm{x}$ also represents the latent vector encoded from the input post. Within the probabilistic learning framework, our main task is to express the conditional probability $p(\mathcal{T}|\bm{x})$ for a pair $(\bm{x}, \mathcal{T})\in D$.
	%\comment{In the probabilistic framework, each node in a tree $\mathcal{T}$ corresponds to a random variable.}
	%In this study, we consider the discrete random variables for tree nodes while the continuous random variables will be considered in our future work.
	We can first reformulate $p(\mathcal{T}|\bm{x})$ as:
	\begin{equation}\label{equ:yxroot}
	\begin{aligned}
	p(\mathcal{T}|\bm{x})=p(t_{\textrm{r}}|\bm{x})\cdot p(\mathcal{T}_{\neg\textrm{r}}\mid t_{\textrm{r}},\bm{x})
	\end{aligned}
	\end{equation}
	where $t_{\textrm{r}}$ and  $\bm{\mathcal{T}}_{\neg\textrm{r}}$ denotes the root and the set of non-root nodes respectively.
	The first term $p(t_{\textrm{r}}|\bm{x})$ in Equ. (\ref{equ:yxroot}) is modeled as follows
	$
	p(t_{\textrm{r}}|\bm{x})=
	\frac{\exp g_{\textrm{r}}(t_{\textrm{r}},\bm{x})}{\sum_{v\in V}\exp g_{\textrm{r}}(v,\bm{x})}
	$
	where $g_\textrm{r}$ is a nonlinear and potentially multi-layered function, and $V$ is the vocabulary containing all possible values for the discrete random variables.% of tree nodes.

	To model $p(\mathcal{T}_{\neg\textrm{r}}\mid t_{\textrm{r}},\bm{x})$, we make the following conditional independence assumption:
	\begin{myApt}\label{as:local}
		The children of different nodes are conditionally independent given their ancestors. \label{apt:ind}
	\end{myApt}
	
	With assumption \ref{as:local}, $p(\mathcal{T}_{\neg\textrm{r}}| t_{\textrm{r}},\bm{x})$ is decomposed as:
	\begin{equation}
	p(\mathcal{T}_{\neg\textrm{r}}| t_{\textrm{r}},\bm{x})
	=
	\prod_{t\in \mathcal{T}} p\big(C(t)\mid \bm{x}, t, {A(t)}\big)
	\end{equation}
	where $C(t)=(c_1,c_2,\cdots,c_K)$ denotes the set of $t$'s children, and $A(t)$ denotes all $t$'s ancestors.

	%Namely, we have
	%\begin{equation}
	%\begin{aligned}
	%&p\big(C(t)\mid \bm{x}, t, {A(t)}\big)\neq\prod_{c_k\in C(t)}p\big(c_k\mid\bm{x}, t, {A(t)}\big)
	%\end{aligned}
	%\end{equation}
	%This observation is motivated by the example in Fig.\ref{fig:childrenofjumps}.
	%
	%
	%For example, as shown in Fig.\ref{fig:childrenofjumps}, in dependency parsing, given the adjective ``\emph{stupid}'', if we observe ``\emph{is}'' as one child node of ``\emph{stupid}'', it is much more likely to obtain ``\emph{does}'' instead of ``\emph{did}'' as another child node. It is based on the fact that the observation of ``\emph{stupid is as stupid does}'' occurs frequently in the world. But if the child were ``\emph{was}'', it would be the other way around.
	
	%\begin{figure}[htbp]
	%	\centering
	%	\subfigure{
	%		\label{fig:foxjumpsdog}
	%		\includegraphics[width=1.2in]{figs/stupid_does.pdf}
	%	}
	%	\subfigure{
	%		\includegraphics[width=1.2in]{figs/stupid_did.pdf}
	%		\label{fig:dolphinjumpsseal}
	%	}
	%	\caption{The children have dependency with each other.}
	%    \label{fig:childrenofjumps}
	%\end{figure}
	We then move to model the conditional probability $p\big(C(t) | \bm{x}, t, {A(t)}\big)$.
	Concretely, since the child nodes to a parent usually correlate with each other, it is inappropriate to assume conditional independence among them.
	Thus, the  probability $p\big(C(t)| \bm{x}, t, {A(t)}\big)$ is then decomposed into the following ordered conditional probabilities:
	\begin{equation}
	\begin{aligned}
	&p\big(C(t)| \bm{x}, t, {A(t)}\big)\!=\!\prod_{c_k\in C(t)}p\big(c_k|\bm{x}, t, {A(t)},  c_{<k}\big)
	\end{aligned}
	\end{equation}
	Furthermore, we argue that children at different positions obtain different underlying meanings. Hence, $K$ different types of hidden states are designed for the $K$ children of node $t$:
	\begin{equation}
	\bm{h}_k=f_k(t,\bm{h},\bm{x}) ~, k=1,2,\dots,K
	\end{equation}
	where $\{f_k\}_{k=1}^K$ are activation functions which can be  LSTM or other RNN cells. $\bm{h}$ denotes the hidden state fed to node $t$, containing the memory from $t$'s ancestors
	%\comment{ and the token at $t$}
	, and $\bm{h}_{r}=\bm{0}$ for the root node.
	With $\bm{h}_k$, we define $p\big(c_k\mid\bm{x}, t, {A(t)},  c_{<k}\big)$ as follows:
	\begin{equation}
	\begin{aligned}
	p\big(c_k\mid\bm{x}, t, {A(t)},  c_{<k}\big)
	&=\!\frac{\exp g_k\big(c_k,\bm{x},t,\bm{h}_k,\tilde{\bm{c}}_{k-1}\big)}{\sum_{v\in V}\exp g_k\big(v,\bm{x},t,\bm{h}_k,\tilde{\bm{c}}_{k-1}\big)}
	\\
	\tilde{\bm{c}}_{k-1}&=[c_1;c_2;\cdots;c_{k-1}]
	\end{aligned}
	\end{equation}
	where $\tilde{\bm{c}}_{k-1}$ is the concatenation of $\{c_i\}_{i=1}^{k-1}$.
	
	Modeling of parent-children dependency is summarized in Fig.\ref{fig:model_graph}. With all these modelings, we train the \textsc{X2Tree} model by maximizing the data likelihood, namely
	\begin{equation}
	\begin{aligned}
	&\prod_{(\bm{x}, \mathcal{T})\in D}p(\mathcal{T}|\bm{x})\\ &
	=\prod_{(\bm{x}, \mathcal{T})\in D}p(t_{\textrm{r}}|\bm{x})
	\prod_{t \in \mathcal{T}}\prod_{c_k\in C(t)}p\big(c_k\mid\bm{x}, t, {A(t)},  c_{<k}\big)
	\end{aligned}
	\end{equation}
	It is worth mentioning that in order to explicitly notify the end of tree generation we need to add the special token ``\textsc{eob}'' (short for ``End Of Branch'') to the leaf nodes as their children. Hence, all the leaf nodes of the tree in the training dataset are \textsc{eob} nodes.
	%with the following model parameters: $\{f_k\}_{k=1}^K$, $\{g_k\}_{k=1}^K$, $g_{\textrm{r}}$ and other model functions. %Here, for different nodes, $\{f_k\}_{k=1}^K$, $\{g_k\}_{k=1}^K$ and $g_{\textrm{r}}$ are shared, as shown in Fig.\ref{fig:sharing}.

	%Note that in certain dataset, the child numbers of tree nodes may be different,
	%while \textsc{X2Tree} model only operates on trees with fixed child number, namely full trees.
	%Meanwhile, it also requires that every tree leaf is a special token ``\textsc{eob}'' (short for ``End Of Branch''), which enables the algorithm to model probabilities of trees with all possible heights.
	%Hence, before training,
	%we fill each input tree with special \textsc{eob} tokens to transform them into full trees in which all leaves are \textsc{eob}s.
	%In detail, for each node $t$, if $t$ obtains $k$ children, we add $(K-k)$ \textsc{eob}s as child nodes to $t$, where $K$ is the maximal child number in the dataset. After this transformation, the trees can be fed to the model.

	%Meanwhile, it also requires that every tree leaf is a special token ``\textsc{eob}'' (short for ``End Of Branch''), which enables the algorithm to model probabilities of trees with all possible heights.
	%Hence, before training,
	%we fill each input tree with special \textsc{eob} tokens to transform them into full trees in which all leaves are \textsc{eob}s.

	\subsection{Tree Canonicalization}\label{sec:tc}
	\begin{figure}
		\centering
		\subfigure[a-b-c]{
			\includegraphics[width=0.7in]{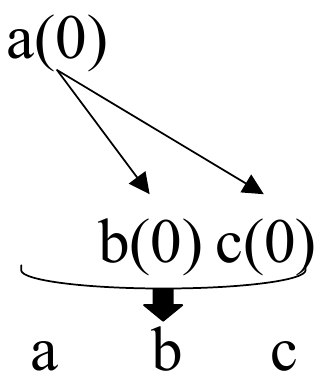}
		}
		\subfigure[b-a-c]{
			\includegraphics[width=0.7in]{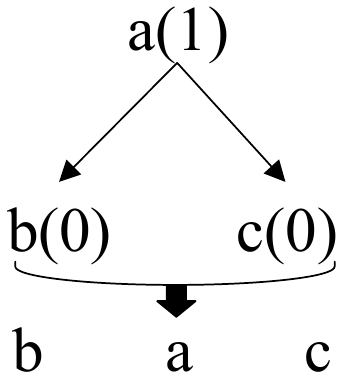}
		}
		\subfigure[b-c-a]{
			\includegraphics[width=0.7in]{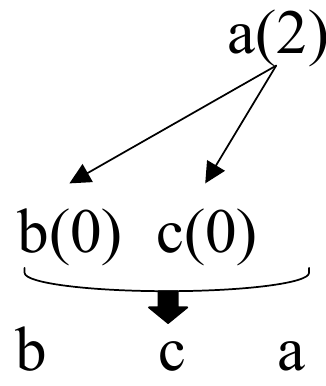}
		}
		\caption{Three different sequence-preserved trees.}
		\label{fig:SPTrees}
	\end{figure}
	
	As aforementioned, the proposed \textsc{X2Tree} model requires that the tree is $K$-ary full tree. Whereas for dialogue generation task a response sentence can be parsed into a dependency tree with any number of child nodes at each level.
	During training and generating, it is difficult to determine the child-node number of a word. Additionally, variable-length data is tricky for GPU acceleration.
	Hence, the original dependency tree is canonicalized into $K$-ary full tree before training.
	%In contrast to the great efforts put into modeling sequences, only a few studies focused on the neural-based modeling of tree structures.

	%our model does not need to decide the number of children while tree generating, because we only need to generate fixed number of children.

	Basically, the transformed $K$-ary full tree should be equivalent to the original one. In other words, there must exist an algorithm to support the bi-directional transformation between a tree and its $K$-ary full counterpart. Considering the number of $K$ is linear to the number of model parameters, to reduce model complexity, we usually hope $K$ to be as small as possible.
	For a given tree, a simple method to transform it into a full tree is to fill all the empty positions with \textsc{eob} nodes. With this method,
	every tree node obtains $K$ children where $K$ is the maximal number of immediate children over all tree nodes.
	%$K$ is the maximal number of immediate children to a tree node.
	However, when $K$ is large and the tree nodes are sparse, the redundant \textsc{eob} nodes significantly increase the learning complexity. Hence, ideally, before the \textsc{eob} filling step we want to transform the tree into a binary or ternary tree.

	%Here a tree canonicalization algorithm is developed to transform a general tree into a full tree,
	%which obtains fixed child number to be the input of \textsc{X2Tree}.
	%In detail, different tree types can be transformed into binary or ternary trees to reduce the space consumption of the resulting trees. %We will also prove the optimality of the tree canonicalization algorithm.
	%Intuitively, to transform a tree into a full tree, we can directly fill the empty positions with \textsc{eob}s, and the resulting tree will be a $K$-ary full tree.
	%However, when $K$ is big and the tree nodes are sparse,
	%redundant \textsc{eob}s will be added.
	%This increase time consumption of local dependency modeling in \textsc{X2Tree}.
	%To address this problem, we firstly transform the tree into a binary or ternary tree:
	%1) For a ordered tree, transform it into a binary tree;
	%2) For a sequence-preserved tree (detailed later) , transform it into a ternary tree.
	%We then fill up the binary or ternary tree with \textsc{eob}s and feed to  \textsc{X2Tree} model.
	
	%\begin{wrapfigure}{R}[0cm]{0pt}
	%	\centering
	%	\includegraphics[width=1.6in,height=0.5in]{figs/ordered_tree.pdf}
	%	\caption{Two different ordered trees.}
	%	\label{fig:ordered_tree}
	%\end{wrapfigure}
	
	\begin{figure}
		\centering
		\includegraphics[width=3in]{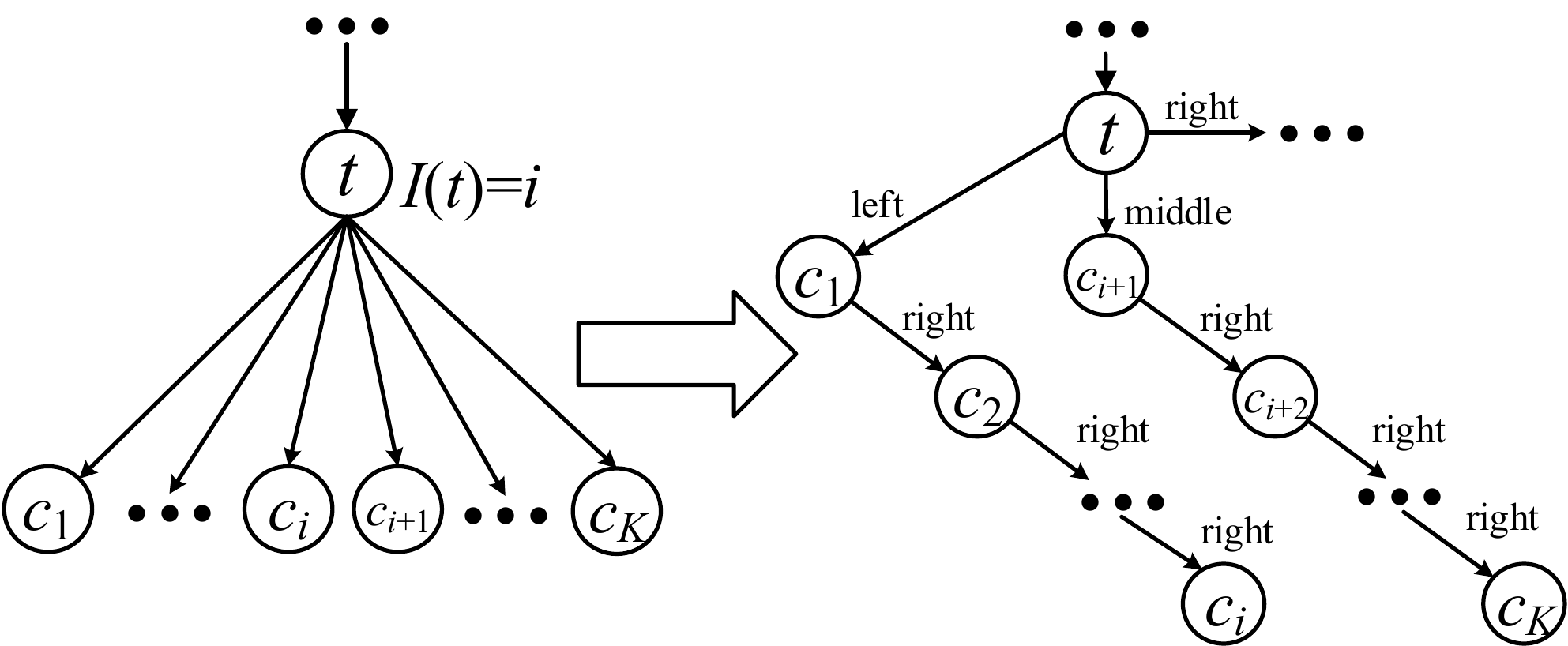}
		\caption{The canonicalization for node $t$.}
		\label{fig:ternary_can}
	\end{figure}

	Here, we mainly consider two scenarios. For an ordered tree, where ordering is specified for the children of each node, we transform it to a left-child right-sibling (LCRS) binary tree~\cite{cormen2009introduction}. This transformation is reversible with a one-to-one mapping between the ordered tree and its LCRS counterpart.
	Furthermore, for the conversational generation tasks, we need to flatten the predicted tree into a sequence. Therefore we need to store position information in the dependency tree. For this purpose we first give the following definition of \emph{sequence-preserved tree} (SP tree for short).

	\begin{myDef}
		An SP tree is an ordered tree where each node $t$ is tagged an integer $I(t)\in\{0,1,\cdots,K\}$, and $K$ is the number of children to node $t$.
	\end{myDef}

	The in-order traversal of an SP tree corresponds to a node sequence. Node $t$'s children are divided into two parts. The left part contains the first $I(t)$ nodes (child nodes are ordered from left to right), while the right part contains the remaining nodes. In the in-order traversal we first visit the nodes in the left  part, then the current node, finally the right part. Fig. \ref{fig:SPTrees} shows three SP trees with their corresponding sequences.
	%Then,  to obtain the corresponding sequence for a SP tree rooted by $t$, we recursively define a function %$s(\cdot)$:
	%\begin{equation}
	%\begin{aligned}
	%s(t)=\!\left\{\!
	%\begin{aligned}
	%&(s_1,\! \cdots, \!s_{I(t)},t, s_{I(t)+1}, \cdots\!, \!s_K)&,&  ~t \text{ is not leaf}\\
	%&t&,& ~ t \text{ is leaf}\\
	%\end{aligned}
	%\right.
	%\end{aligned}
	%\end{equation}
	%where $s_k=s(c_k)$.
	%$t$ is on the right of $c_{I(t)}$ and their descants
	%and on the left of $c_{I(t)+1}$ and their descants.
	%Note that the subtrees rooted by $\{c_k\}_{k=1}^{I(t)}$ are all on the left of $t$, and the subtrees rooted by $\{c_k\}_{k=I(t)+1}^{K}$ are all on the right
	Obviously, the dependency tree of a sentence is an SP tree, where the number attached on each node can be obtained by checking the position relationship of the node and its children in the original sentence, as shown in Figure~\ref{fig:dependency_parser}. For example, the node ``says'' obtains a number ``$1$'' which means one child of this node are on its left part in the original sequence.
	%\subsubsection{Canonicalization of SP Trees}\label{sec:sptcan}
	As discussed earlier, a tree canonicalization step is needed to transform the original dependency tree into a $K$-ary full tree. To preserve sequence order, we transform the dependency into a ternary tree. We now present the algorithm and discuss why ternary tree is the ``best" choice.
	Alg.~\ref{alg:tranformTriple} details this canonicalization process, and an illustration is shown in Fig.\ref{fig:ternary_can}.
	
	In a ternary tree, each node has three children, namely left, middle and right nodes. For node $t$ with attached number $I(t)$, Alg.~\ref{alg:tranformTriple} first determines its left and middle child in the ternary tree. Specifically, its left child is set to $c_1$, the first child in the original tree; and its middle child is set to $c_{I(t)+1}$. Any other child $c_j \big(j\neq 1\text{ and }j\neq I(t)+1\big)$ is set as the right child of $c_{j-1}$ recursively. With this ternary tree a simple in-order traversal in the order of left child, parent, middle child and right child can restore it into a sequence.

	\begin{algorithm}[H]\small
		\caption{$\textsc{Canonicalize}$}
		\label{alg:tranformTriple}
		\begin{algorithmic}[1]
			\REQUIRE A node of SP tree, $t$\\
			\ENSURE Ternary tree node corresponding to  $t$, $t'$\\
			\STATE Let $c_1,c_2,\cdots,c_n$ denote $t$'s children;
			\STATE Create an new node $t'=t$;
			\FOR{$j\leftarrow 1$ {\bfseries to} $n$}\label{step:leftfor}
			\STATE currentNode $\leftarrow\textsc{Canonicalize}(c_j)$;
			\IF{$j=1$ \AND $j\leqslant I(t)$}
			\STATE  $t'.\textrm{leftChild} \leftarrow$currentNode;
			\ELSIF{$j=I(t)+1$}
			\STATE  $t'.\textrm{middleChild} \leftarrow$currentNode;
			\ELSE
			\STATE  lastNode$.\textrm{rightChild}\leftarrow$currentNode;
			\ENDIF
			\STATE lastNode $\leftarrow$ currentNode;
			\ENDFOR\label{step:leftforend}
			\STATE {\bfseries return} $t'$;
		\end{algorithmic}
	\end{algorithm}

	Next, we prove that the resulting ternary tree is equivalent to the original SP tree in the sense that they can be transformed into each other.
	\begin{myTheo}\label{tho:1}
		Given any SP tree $\mathcal{T}$, it can be transformed into a ternary tree $\mathcal{T}'$, and $\mathcal{T}'$ can be transformed back into the original tree $\mathcal{T}$.
	\end{myTheo}
	
	\begin{proof}
		Using the Alg.2, we can transform $\mathcal{T}$ into a ternary tree $\mathcal{T}'$.
		
		We now show how to transform $\mathcal{T}'$ back into $\mathcal{T}$. For each node $t\in\mathcal{T}'$, if $t$ is not a right child, let $r_1$ denote the right child of $t$, $r_2$ denote the right child of $r_1$, $r_n$ denote the right child of $r_{n-1}$ until $r_n$ obtains no right child.
		% (if $\mathcal{T}$ is $K$-ary tree, $r_k$ obtains no right child).
		
		In the original tree $\mathcal{T}$, $t$ and $\{r_j\}_{j=1}^n$ must be siblings. For simplicity, let $t'$ denote their parent.
		
		1) If $t$ is a left child in $\mathcal{T}'$,  ($t$ , $r_1$ , $\cdots$, $r_n$) are
		first, second, \dots,  $(n+1)$-th child of $t'$ in the original SP tree $\mathcal{T}$.
		
		2) If $t$ is a middle child in $\mathcal{T}'$, ($t$ , $r_1$ , $\cdots$, $r_n$) are $\big(I(t')+1\big)$-th, $\big(I(t')+2\big)$-th, \dots, $\big(I(t')+n+1\big)$-th child of $t'$ in the original SP tree $\mathcal{T}$.
		
		In this way, for each node in $\mathcal{T}'$, we can find its original position in $\mathcal{T}$, and then re-converts $\mathcal{T}'$ to  $\mathcal{T}$.
		
	\end{proof}

	\begin{figure*}[htbp]
		\centering
		\subfigure[The original tree ``I-am-a''.]{
			\label{fig:onestepGBS1}
			\includegraphics[width=1.2in]{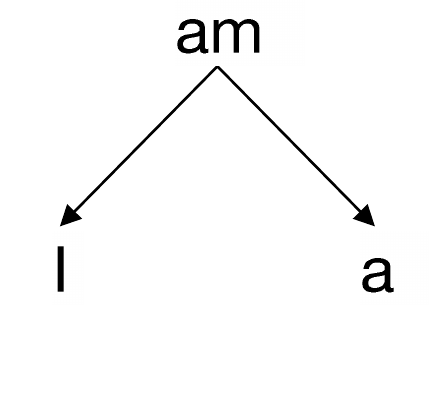}
		}
		\subfigure[Extending the leave node ``I''.]{
			\label{fig:onestepGBS2}
			\includegraphics[width=2.6in]{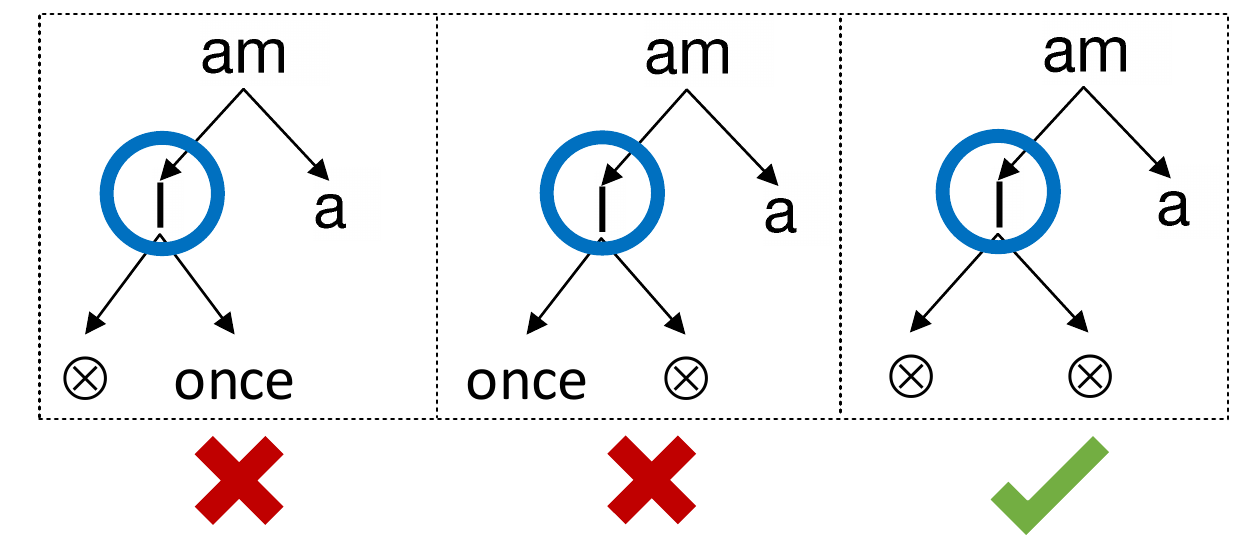}
		}
		\subfigure[Extending the leave node ``a''.]{
			\label{fig:onestepGBS3}
			\includegraphics[width=2.6in]{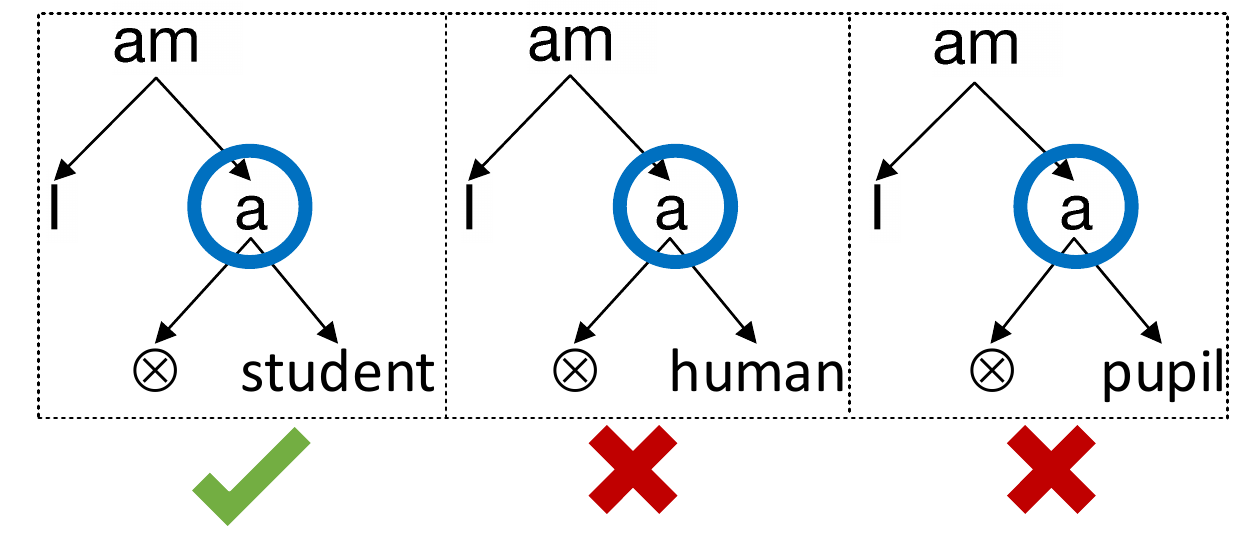}
		}
		\caption{Examples of one step of generalized beam search. The Fig.(a) shows the original tree. The Fig. (b) and (c) show the searching results. Note that words in double quotes are to be expanded.  Here, ``$\otimes$'' denotes special token ``\textsc{eob}''.}\label{fig:onestepGBS}
	\end{figure*}

	%From Theorem~\ref{tho:1}, an SP tree $\mathcal{T}$ can be transformed into an ternary tree $\mathcal{T}'$. Note that SP tree $\mathcal{T}$ corresponds to a sequence.
	%To directly transform $\mathcal{T}'$ into the sequence,
	%we can  traverse $\mathcal{T}'$ in the order of: left child, parent, middle child and right child.
	
	%Here, a trick can be applied to directly transforming $\mathcal{T}'$ into the sequence. It is to traverse $\mathcal{T}'$ in the order of: left child, parent, middle child and right child.

	%The new ternary tree preserves the relative position between one node and its children in the original sequence. To transform a ternary dependency tree back into the original sequence, we can traverse it in the order of: left child, parent, middle child and right child.
	%It is also possible to transform any dependency tree into a binary tree. However, relative positions of the tokens in the original sequence are lost during transformation.
	
	Additionally, we prove that ternary tree is the ``best" choice for model complexity. Theoretically, a dependency tree is equivalent to a $K$-ary tree when $K\geqslant 3$. Since the number of $K$ is linear to parameter size in the \textsc{X2Tree} model, we prefer simpler models with smaller values of $K$. Theorem~\ref{tho:2} formally shows that SP trees are \emph{not} equivalent to binary trees. Therefore, the ternary tree is the ``best" choice.
	Thus before training, we perform a preprocessing step which converts each response into its corresponding dependency tree (instance of SP tree), and canonicalize them into ternary trees.
	%Intuitively, it is also possible to transform an SP tree into a binary tree. However, relative positions of the tokens in the original sequence may lose during transformation. Here, we give Theorem \ref{tho:2}:
	\begin{myTheo}\label{tho:2}
		Given any SP tree $\mathcal{T}$, no algorithm exists which transforms $\mathcal{T}$ into an LCRS tree $\mathcal{T'}$ and re-converts $\mathcal{T'}$ to $\mathcal{T}$.
	\end{myTheo}	
	\begin{proof}
		Let $\mathcal{S}^n$, $\mathcal{O}^n$ and $\mathcal{L}^n$ respectively denote the set of sequence-preserved trees, ordered trees and LCRS trees with $n$ nodes.
		Since the ordered trees and LCRS trees obtain one-to-one correspondence~\cite{cormen2009introduction},
		it can be inferred that the element number $|\mathcal{O}^n|=|\mathcal{L}^n|$.
		
		For a node $t$ in ordered tree, if $t$ and its children obtain specified ordering, namely $I(t)$ is defined, it converts to a SP tree.
		Furthermore, for different $I(t)$, the SP trees are different.
		Thus, $|\mathcal{S}^n|>|\mathcal{O}^n|=|\mathcal{L}^n|$.
		Moreover, suppose that an algorithm exists that transforms $\mathcal{T}$ into an LCRS tree $\mathcal{T}'$, and re-converts $\mathcal{T'}$ to $\mathcal{T}$. This infers that $|\mathcal{S}^n|\leqslant|\mathcal{L}^n|$. It is contradictory to $|\mathcal{S}^n|>|\mathcal{L}^n|$.
		
	\end{proof}
	
	Note that the generated tree is a full tree (the leaf nodes can be in the lower depth) but not a perfect tree. For a sentence with n words, the transformed k-ary full tree contains exactly $(kn+1)$ nodes, and the extra $(kn+1-n)$ nodes are the EOB tokens. Thus, only the $(kn+1-n)$ nodes induce the computing waste. To minimize this waste we expect k as small as possible. Theorems 1 and 2 tell that $k=3$ is the minimal number we can use so that the transformed tree is equivalent to the original dependency tree. 
	
	%Therefore, in our application we canonicalize the SP trees into ternary trees for \textsc{X2Tree} model training.

	%a tree $\mathcal{T}$, if it is a SP tree, canonicalize it into a ternary tree.
	%If it is a ordered tree, canonicalize it into a binary tree.
	%Additionally, the time complexity of the transformation is linear w.r.t. the length of the sequence. Hence, the time consumption is acceptable.
	
	\subsection{Tree Generation}\label{sec:tg}
	
	With the trained model we can infer the most probable trees for a given input $\bm{x}$. In this section we develop a greedy search algorithm for this inference task.

	The beam search is traditionally adopted for sequence structure generation
	At each step, it keeps $G$ (called \emph{global} beam size) best candidates with the maximal probabilities so far. Then, only those candidates are expanded next. For each candidate on the beam it grows a new node at the current end of the sequence. This process repeats recursively until all candidates end with \textsc{eob} nodes. %The whole search process can be represented by a search tree, which is visualized in the slides provided in the supplementary files.

	%With the trained \textsc{X2Tree} network,
	%we develop the following generalized beam search to generate
	%responses for an input $\bm{x}$.
	%We first briefly introduce the beam search~\cite{Sutskever2014} which is a popular algorithm for sequential generation.
	%A step of tree searching is shown in Fig.\ref{fig:beam_search}.
	%At every step, the algorithm generate $L=3$ ($L$ is called Local Beam Size) for each candidate,
	%and retain only $G=2$ ($G$ is called Global Beam Size) newly generated candidates for the next step.
	%Finally, the algorithm outputs $G=2$ sequences with highest probabilities over all candidates.
	
	%\begin{figure}[htb]
	%	\centering
	%	\includegraphics[width=2.5in]{figs/beam_search.pdf}
	%	\caption{An example of search tree generated by the beam search, in which Global Beam Size $G=2$ and Local Beam Size $L=3$. Here, \textsc{eos} denotes the special token ``End-Of-Sentence''.}
	%	\label{fig:beam_search}
	%\end{figure}

	\begin{algorithm}[H]\small%\setlength{\textfloatsep}{1\baselineskip}
		\caption{$\textsc{GeneralizedBeamSearch}$}
		\label{alg:Expand}
		\begin{algorithmic}[1]
			\REQUIRE ~~\\
			latent vector, $\bm{x}$ \\
			global beam size, $G$ \\
			local beam size, $L$ \\
			child number of each node, $K$ \\
			\ENSURE A set of trees, $\mathcal{R}$
			\STATE $\mathcal{S}$ $\leftarrow$ \{$G$ roots with highest $p(t_{\textrm{r}}|\bm{x})$\}; $\mathcal{R}\leftarrow\phi$;
			\WHILE{$|\mathcal{R}|<G$}
			\FOR{\textbf{each} $\mathcal{T}\in\mathcal{S}$}
			\FOR{\textbf{each} leaf $t \in \mathcal{T}$}
			\IF{$t = \textsc{eob}$}
			\STATE \textbf{continue};
			\ENDIF
			\STATE Via chain beam search find $L$ groups of $C(t)$ by maximizing $p\big(C(t)|\bm{x}, t, {A(t)}\big)$;
			\FOR{\textbf{each} $C(t)$}
			\STATE Connect $C(t)$ to $\mathcal{T}$ as a new tree $\mathcal{T}'$;
			\STATE Add $\mathcal{T}'$ to $\mathcal{S}$;
			\ENDFOR
			\ENDFOR
			\STATE Delete $\mathcal{T}$ from $\mathcal{S}$;
			\ENDFOR
			\STATE $\mathcal{S} \leftarrow$  \{$G$ trees  with highest $p(\mathcal{T}|\bm{x})$ in $\mathcal{S}$\};
			\FOR{\textbf{each} $\mathcal{T} \in \mathcal{S}$}
			\IF{$\mathcal{T}$'s leaves are all \textsc{eob}s}
			\STATE Add $\mathcal{T}$ to $\mathcal{R}$;
			\ENDIF
			\ENDFOR
			\ENDWHILE
			\STATE $\mathcal{R} \leftarrow$  \{$G$ trees  with highest $p(\mathcal{T}|\bm{x})$ in $\mathcal{R}$\};
			\STATE {\bfseries return} $\mathcal{R}$;
		\end{algorithmic}
	\end{algorithm}

	Since sequence is a special case of trees, searching tree generation has more challenges to address. First, an arbitrary tree has multiple leaves which could potentially generate new children. Second, when growing new children for a leaf node we need to generate all children as a whole since they correlate with each other (as mentioned in Section \emph{Generative Model for $K$-ary Full Tree}). Multiple groups of such $K$ children need to be generated as the best candidates. %The original beam search method for sequence generation can be used here.
	
	We use the example in Fig.~\ref{fig:onestepGBS} to describe this tree generation method. The original tree has two leaves, nodes ``\emph{i}" and ``\emph{a}". For each of these leaves, we can generate new children. Specifically, for node ``\emph{i}" it generates $L$ groups of $K$ children, as shown in Fig.~\ref{fig:onestepGBS2} ($L=3$ and $K=2$ in this example). Since these new children are ordered, this local step of children generation is actually a task of sequence generation, thus the conventional beam search can be used. Here, $L$ (called \emph{local} beam size) is to specify the number of candidate sequences generated for each leaf. After the child generation for all the leaves, we compare all these candidate trees and only retain top-$G$ ($G=2$ in this example) trees for the next round of generation. This process recursively continues until all the leaves in the tree are \textsc{eob} nodes.
	Note that the proposed method is a generalized beam search.
	Beam search for sequence generation is a special case with $K=1$, since sequence is equivalent to $1$-ary tree.
	The method is detailed in Algorithm~\ref{alg:Expand}.
	%A visualization of this search process is provided in the slides in the supplementary files.
	
	\section{Experiment Settings}
	%We evaluate the proposed model and compare it with four state-of-the-art neural-based models on an open domain conversation dataset.
	
	\subsection{Dataset Details}\label{sec:dataset}
	Our experiments focus on dialogue generation task. 14 million post-response pairs were obtained from Tencent Weibo\footnote{http://t.qq.com/?lang=en\_US}. After removing spams and advertisements, $815,852$ pairs were left, among which $775,852$ are for training, and $40,000$ for model validation.
	
	\subsection{Benchmark Methods}
	We implemented the following four popular neural-based dialogue models for comparison:
	\begin{enumerate}
		\item  \textsc{Seq2Seq}\cite{Sutskever2014}: A RNN model that utilizes the last hidden state of the encoder as the initial hidden state of the decoder;
		\item  \textsc{EncDec}\cite{Cho2014}: A RNN model that feeds the last hidden state of the encoder to every cell and softmax unit of the decoder;
		\item  \textsc{ATT}\cite{Bahdanau2014Neural}: A RNN model based on \textsc{EncDec} with attention signal;
		\item  NRM\cite{Shang2015}: Neural Responding Machine with both global and local schemes.
		%5) MMMI-bidi and MMI-antiLM~\cite{Li2015}: The one-layer encoder-decoder model using Maximum Mutual Information (MMI) as the objective function to reorder generated responses.
	\end{enumerate}
	All these models map sequences to sequences directly,
	and only differ in how to summarize the encoder hidden states into a latent vector.
	Thus, the proposed tree decoder can be applied to any of these models, and potentially improve the response quality from a different perspective.
	%In this study, a tree decoder grafted on \textsc{EncDec}~\cite{Cho2014} is implemented for evaluation (denoted as \textsc{X2Tree}).
	Here, we stress that this tree-decoder can be easily applied to the model~\cite{Serban2015}, which summarizes multiple rounds of dialogues into a latent vector. In the future, tree decoder for multi-round dialog will be evaluated.
	
	\subsection{Implementation Details}
	%For \textsc{X2Tree} model, LTP dependency parser is utilized to transform the responses into dependency trees before training.
	All sentences in the experiments are segmented by LTP\footnote{https://www.ltp-cloud.com/intro/en/}. A vocabulary of 28,000 most frequent Chinese words in the corpus is used for training, which contains 97\% words.
	%This vocabulary covers $99.99\%$ of the words in the corpus.
	Out-of-vocabulary words are replaced with ``\textsc{unk}''.
	%Note that \textsc{eob} in triple dependency tree is treated as a word like \textsc{unk} and other words.
	Our implementations are based on the Theano library~\cite{bastien2012theano} over NVIDIA K80 GPU.
	We applied one-layer GRU~\cite{Cho2014} with 1,024-dimensional hidden states to $\{f_k\}_{k=1}^K$ and all baseline models. As suggested in~\cite{Shang2015}, the word embeddings for the encoders and decoders are learned separately, whose dimensions are set to 128 for all models.
	%Some initial experiments demonstrated that two separate sets1 of word embeddings improve the performances.
	All the parameters were initialized using a uniform distribution between -0.01 and 0.01. In training, the mini-batch size is $128$. We used ADADELTA~\cite{Zeiler2012ADADELTA} for optimization. The training stops if the perplexity on the validation set increases for 4 consecutive epochs. Models with best perplexities are selected for further evaluation.
	When generating responses, for \textsc{X2Tree} we use generalized beam search with global beam size $G=6$, local beam size $L=6$.
	For other $\textsc{X2Seq}$ baseline models, conventional beam search with beam size $200$ is used.
	
	%\item Models are trained on the NVIDIA Tesla K40C GPU, based on Theano\cite{bastien2012theano} library .
	
	\subsection{Evaluation Methods}
	Due to the high diversity nature of dialogs, it is practically impossible to construct a data set which adequately covers all responses for each given post.
	%Thus, match-based metrics, for example the BLEU ~\cite{Papineni2002}, are not appropriate. While perplexity is widely used in SMT, lower values of this measure do not lead to better responses~\cite{Liu2016}.
	Hence, we apply human judgment to our experiments. In detail, 3 labelers were invited to evaluate the quality of responses to $300$ randomly sampled posts.
	For each post, each model generated top-$5$ different responses (for a total of $25$). For fair comparison, we create a single file in which each post is followed by its $25$ responses which are shuffled to avoid labelers knowing which model each response is generated by.
	
	For each response the labelers determine the quality to be one of the following three levels:
	\begin{itemize}[noitemsep,topsep=0pt,parsep=0pt,partopsep=0pt]
		\item \textbf{Level 1}: The response is ungrammatical.
		\item \textbf{Level 2}: The response is basically grammatical but irrelevant to the input post.
		\item \textbf{Level 3}: The response is grammatical and relevant to the input post. The response on this level is acceptable for dialog system.
	\end{itemize}
	From labeling results, average percentages of responses in different levels are calculated. Additionally, labeling agreement is evaluated by Fleiss' kappa~\cite{Fleiss1971Measuring} which is a measure of inter-rater consistency.
	Furthermore, we also report BLEU-4~\cite{Papineni2002} scores for these 300 posts.
	%, which is conventionally applied in translation tasks.
	Since some researchers indicate BLEU may not be a good measure for dialog evaluation\cite{Liu2016}, we consider human judgment as a major measure in experiments. 
	%At last, to compare the GPU acceleration performance, we report the time consumption of  per batch in the unit of second.
	
	\subsection{Experimental Results and Analysis}
	
	The experimental results are summarized in Table \ref{table:human_judge}. 
	For \textsc{Seq2Seq}, \textsc{NRM} and \textsc{X2Tree}, the agreement value is in a range from 0.6 to 0.8 which is interpreted as ``substantial agreement''.
	Meanwhile, \textsc{EncDec} and \textsc{ATT} obtain a relatively higher kappa value between 0.8 to 1.0 which is ``almost perfect agreement''.
	Hence, we believe the labeling standard is considered clear which leads to high agreement among labelers.
	
	\begin{table}[htbp]\small\setlength{\tabcolsep}{1.2pt}
		\centering
		\caption{The results from human judgment.}\label{table:human_judge}
		\begin{tabular}{l||c|c|c||c||c}
			\hline
			\hline		
			\textbf{Models} &  {Level-1}\% & {Level-2}\% & {Level-3}\% & {Agreement} & BLEU\\
			\hline
			\textsc{EncDec} & 0.44	& 58.89 &	40.67 & 0.8114 & 8.78\\
			\textsc{Seq2Seq} & 1.58 &	50.73 &	47.69 & 0.7834 &12.45\\
			\textsc{ATT} & 2.31 &	45.31	& 52.38 & 0.8269 & 13.89\\
			\textsc{NRM} & 0.64 &	44.98 &	54.38 & 0.7809 & 13.73\\
			\hline
			\textsc{X2Tree} & 0.44 &	34.02 &	65.53 & 0.7733 &15.87 \\
			\hline
		\end{tabular}
	\end{table}
	
	% HERE I WILL TALK ABOUT THE -------ACCEPTABLE---------
	For the Level-3 (acceptable ratio), \textsc{X2Tree} visibly outperforms other models. The best baseline method \textsc{NRM} achieves 54.38\% Level-3 ratio, while \textsc{X2Tree} reaches 65.53\% with an increase percentage of 11.15\%. This improvement is mainly due to less irrelevant (Level-2) responses being generated (34.02\% v.s. 44.98\%), indicating \textsc{X2Tree} outputs more acceptable responses.

	% HERE I WILL TALK ABOUT THE -------GRAMMATICAL---------
	We further notice from Table \ref{table:human_judge} that the percentage of ungrammatical (Level-1) responses from \textsc{X2Tree} is less than other baselines (equal to \textsc{EncDec}) and the BLEU score is greater than other baselines in the experiments. 
	%We believe that the boost of \textsc{X2Tree} Normal and Good ratio mainly comes from improvement of the response grammar. In other words, 
	It shows that responses generated by the tree-structured decoder are more grammatical than those from the chain-structured decoders and demonstrate the \textsc{X2Tree}'s robustness to parser errors.
	%Also, for the dialog generation task the final output is a sentence. The parsed dependency tree is only an immediate result for this task and obviously contains some errors. If the parsed errors occur in the similar patterns, the X2Tree model can learn this ``error pattern''. After we convert the generated tree into a sequence, the sequence may be still correct. Thus, we argue that X2Tree may not be sensitive to parsing errors if they occur in the similar patterns.
	Additionally,  \textsc{X2Tree} and \textsc{EncDec} achieve best grammatical ratio (99.56\%), but \textsc{EncDec} fails in generating relevant responses. Hence, Tree Decoder can improve the response relevance in experiments. We conjecture the reason is that \textsc{X2Tree} firstly generate the core verb of the responses. The first generated may be more relevant to the post and makes the whole response more relevant to the post.
	
	%For time consumption, \textsc{X2Tree} (3.312) is less than the other baselines except for the \textsc{Seq2Seq} method (3.311) which is sightly faster. We believe this is due to that the tree canonicalization improves GPU performance. The exception of \textsc{Seq2Seq} may be due to that \textsc{Seq2Seq}'s parameter amount is much less than \textsc{X2Tree}'s, which can be faster to update.
	%%(only one activate function for \textsc{Seq2Seq} v.s. $K$ activate functions $\{f_k\}_{k=1}^K$ for \textsc{X2Tree}).
	%Also, we note that the time consumption may be different for different versions of Theano optimizers.

	In summary, the experiments demonstrate that \textsc{X2Tree} is able to generate more grammatical and relevant responses, and also show \textsc{X2Tree} obtains the ability to generate correct trees.

	\subsection{Easiness to Learning}
	From Table \ref{table:human_judge}, we discover that the percentage of grammatical responses from \textsc{X2Tree} visibly surpasses other models in the experiments.
	We conjecture that the tree-structured decoder is easier to learn because its hidden states need to store less information than their counterparts in a chain-structured decoder.
	
	In detail, given a response utterance with length $T$, the hidden state at position $t$ in a chain-structured decoder needs to store the information of all previous words $\bm{y}_{<t}$, the average size of $\bm{y}_{<t}$ is $\frac{T+ 1}{2}$ (with an extra \textsc{eos} token).
	In contrast, in a tree-structured decoder, $\bm{h}_t$ only needs to store the information of its ancestors $y_t$. After transforming the response into a triple dependency tree structure, the average depth of nodes is $O(\sqrt{T})$ \cite{Flajolet1982}. In the worst case, the depth of a triple dependency tree is $T$, and the average number of ancestors of nodes is $\frac{T+1}{2}$, which is the same to the average size of $\bm{y}_{<t}$. Fig. \ref{fig:steps_to_store} shows the average number of steps hidden states need to remember at different sequence lengths for our data set. 
	
	\begin{figure}[htbp]
		\begin{center}
			\includegraphics[width=3.2in]{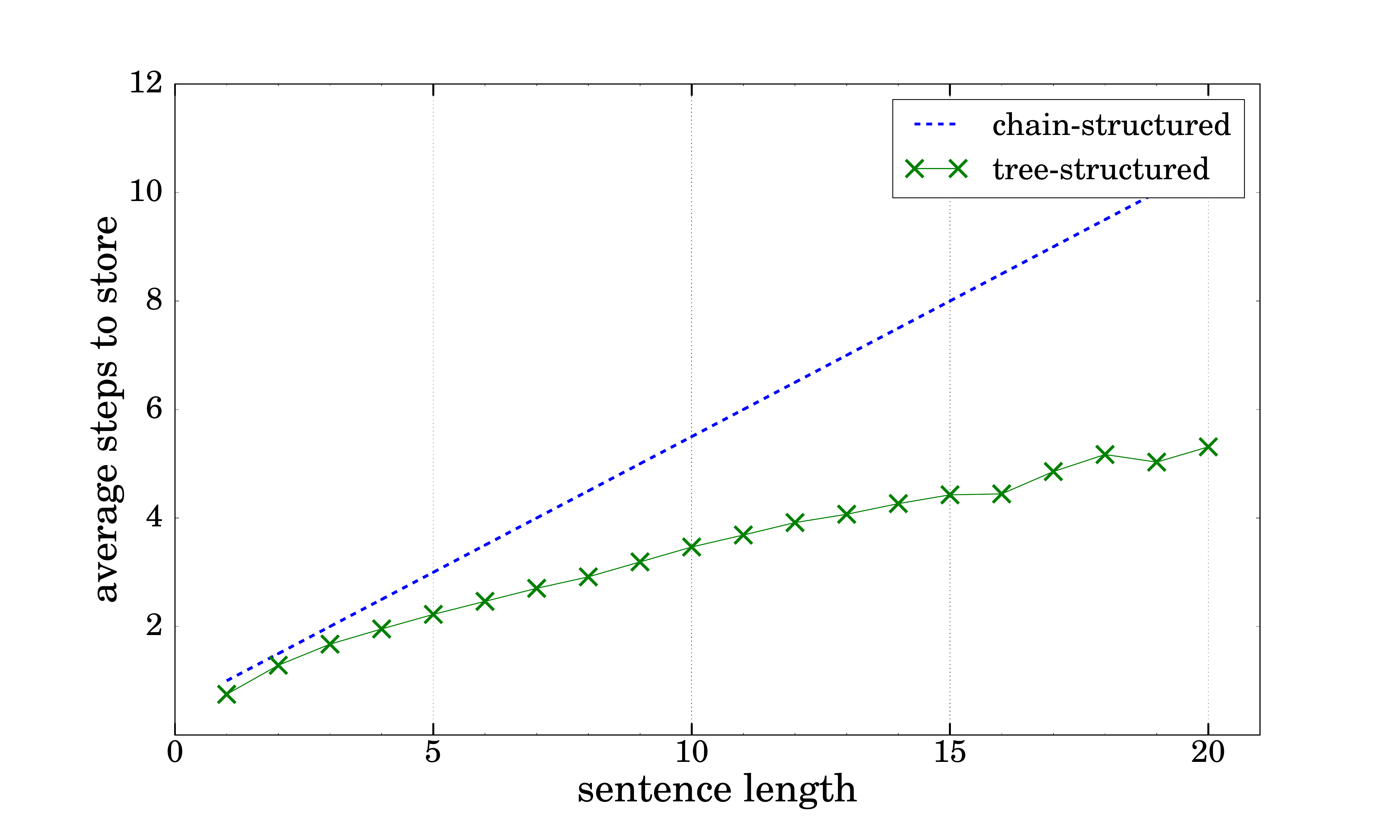}
			\caption{Average number of steps need to be stored for hidden states in both structures.}
			\label{fig:steps_to_store}
		\end{center}
	\end{figure}
	
	Overall, hidden states of a tree-structured decoder need store less information than chain-structured decoder's. This makes \textsc{X2Tree} potentially capable to handle more complex semantic structures in the response utterances.
	
	\section{Related Work}\label{sec:relatedWork}
	%`There are mainly three types of  studies related to this work:
	
	%In this section we discuss some of the related work to Seq2Tree model for conversation.
	
	\textbf{Statistical Machine Translation.} The neural-based encoder-decoder framework for generative conversation models follows the line of statistical machine translation.
	Sutskever et al.~\shortcite{Sutskever2014} used multi-layered LSTM as the encoder and the decoder for machine translation.
	Later, Cho et al.~ \shortcite{Cho2014} proposed the encoder-decoder framework, where the context vector is fed to every unit in the decoder.
	Bahdana et al.~\shortcite{Bahdanau2014Neural} extended the encoder-decoder framework with the attention mechanism to model the alignment between source and target sequences. %It improves the performance of translation on long input sentences.
	%All these studies do not utilize the information of dependency relations.

	\textbf{Conversation models.}
	Inspired by neural SMT, recent studies showed that these models can also be successfully applied to dialogue systems.
	Specifically, for short conversation, Shang et al.~\shortcite{Shang2015} proposed the Neural Responding Machine which further extended the attention mechanism with both global and local schemes.
	Zhou et al.~\shortcite{Zhou2017MARM} proposed MARM to generate diverse responses upon multiple mechanisms.
	Most recently, some researchers focused on multi-round conversation.
	Serban et al.~\shortcite{Serban2015} built an end-to-end dialogue system using hierarchical neural network.
	Sordoni et al.~\shortcite{Sordoni2015} proposed a related model with a hierarchical recurrent encoder-decoder framework for query suggestion.
	%The basic idea for multi-round conversation is to extend the context generation from the immediate previous sentence to several previous ones.
	Our proposed model can also be applied to these multi-round conversation models and potentially improve the performances.
	
	\textbf{Tree-Structured Neural Network.}
	Recently, some studies use tree-structured neural network instead of the conventional chain-structured neural network to improve the quality of semantic representation.
	Socher et al.~\shortcite{Socher2013Recursive} proposed the Recursive Neural Tensor Network. Each phrase is represented by word vectors and its parse tree. Vectors of higher level nodes are computed using their child phrase vectors.
	Tai et al. ~\shortcite{Tai2015} and Zhu et al.~\shortcite{Zhu2015Long} extended the chain-structured LSTM to tree structures. %which composes its state from an input vector and the hidden states of child units.
	All above models use tree structures to summarize a sentence into a context vector, while we propose to decode from a context vector to generate sentences in a root-to-leaf direction.
	%Addtionally, Zhang et al.\shortcite{Zhang2016Top} proposed Tree LSTM activation functions modeling the probability of denpendency trees and apply the model to sentence completion and dependency parsing reranking tasks.
	%Note that
	%some works model trees in a \emph{top-down} fashion and with non-fixed child number. For example,
	Additionally, Zhang et al.\shortcite{Zhang2016Top} proposed Tree LSTM activation function in \emph{top-down} fashion.
	Here, two important points differentiate our work with theirs.
	First, Zhang et al. mainly estimate generation probability of dependency tree and apply their model to sentence completion and dependency parsing reranking tasks,
	while \textsc{X2Tree} handles dialogue modeling in encoder-decoder framework.
	%Firstly, Zhang et al. mainly handle the dependency tree,
	%while \textsc{X2Tree} handles different types of trees, including:
	%a)	tree with non-fixed child number vs. tree with fixed child number;
	%b)	ordered tree vs. sequence-preserved tree (defined later).
	%Via the tree canonicalization method, we can transform these various trees into $K$-ary full trees. Thus, we handle various types of trees in a unified framework.
	%In Zhang et al.¡¯s work, they mainly handle the dependency tree (a special instance of the sequence-preserved tree, defined in our paper). Additionally, we give some theoretical analysis (Theorem 1 in Section \ref{sec:sptcan}) to show the optimum of the tree after canonicalization.
	Second, due to the canonicalization method, \textsc{X2Tree} model process fixed  number ($K=3$) of children at each step for GPU acceleration, while Zhang et al. need to process the children sequentially. Thus, the proposed tree canonicalization method helps to reduce the training time.
	To this end, some works also aim at generating different structure types.
	Rabinovich et al.~\shortcite{Rabinovich2017} proposed the abstract syntax networks to transform card image of the game HearthStone  into well-formed and executable outputs.
	Cheng et al.~\shortcite{Cheng2017} utilized predicate-argument structures to store natural language utterances as intermediate and domain-general representations.

	\section{Conclusion and Future Work}\label{sec:con}
	
	In this study, we proposed a tree-structured decoder to improve the response quality in dialogue systems. By incorporating linguistic knowledge into the modeling process, the proposed \textsc{X2Tree} framework outperforms baseline methods over 11.15\% increase of acceptance ratio in response generation. Future study on incorporating a tree-structured encoder is promising to further enhance the sentence generation quality.

	\section{Acknowledgments}
	This work was supported by the National Key Research and Development Program of China under Grant No. 2017YFB1002104, the National Natural Science Foundation of China (No.61473274, 61573335).
	
	This work was also supported by WeChat Tencent. We thank  Leyu Lin,  Lixin Zhang, Cheng Niu and Xiaohu Cheng for their constructive advices. We also thank the anonymous AAAI reviewers for their helpful feedback.    
 
	\bibliographystyle{aaai}
	\small{
		\bibliography{docs}
	}
	
\end{document}